\RequirePackage{amsmath}
\documentclass[runningheads]{llncs}
\usepackage[T1]{fontenc}
\usepackage[utf8]{inputenc}
\usepackage{graphicx}
\usepackage{booktabs}
\usepackage{ulem}
\usepackage{amssymb}
\usepackage{amsmath}
\usepackage{hyperref}
\usepackage{cleveref}
\usepackage{todonotes}
\usepackage{siunitx}
\usepackage[frak=esstix]{mathalpha}
\usepackage{xr}
\makeatletter

\newcommand*{\addFileDependency}[1]{
\typeout{(#1)}
\@addtofilelist{#1}
\IfFileExists{#1}{}{\typeout{No file #1.}}
}\makeatother

\newcommand*{\myexternaldocument}[1]{%
\externaldocument{#1}%
\addFileDependency{#1.tex}%
\addFileDependency{#1.aux}%
}
\myexternaldocument{supplementary}

\newcommand*\diff{\mathop{}\!\mathrm{d}}
\newcommand{\repourl}{\url{https://anonymous.4open.science/r/grind-FFCB}}

\DeclareUnicodeCharacter{2061}{ }
\DeclareUnicodeCharacter{2212}{-}
\DeclareUnicodeCharacter{03B2}{$\beta$}

\begin{document}
\title{GrINd: Grid Interpolation Network for Scattered Observations}
%
%
\author{Andrzej Dulny\and
Paul Heinisch\and
Andreas Hotho\and
Anna Krause}
%
\authorrunning{Dulny et al.}
%
\institute{CAIDAS, University Würzburg, Germany\\
\email{andrzej.dulny@uni-wuerzburg.de}\\
\email{paul.heinisch@stud-mail.uni-wuerzburg.de}\\
\email{hotho@informatik.uni-wuerzburg.de}\\
\email{anna.krause@uni-wuerzburg.de}
}

\maketitle              
\begin{abstract}
Predicting the evolution of spatiotemporal physical systems from sparse and scattered observational data poses a significant challenge in various scientific domains. Traditional methods rely on dense grid-structured data, limiting their applicability in scenarios with sparse observations. To address this challenge, we introduce GrINd (Grid Interpolation Network for Scattered Observations), a novel network architecture that leverages the high-performance of grid-based models by mapping scattered observations onto a high-resolution grid using a Fourier Interpolation Layer. In the high-resolution space, a NeuralPDE-class model predicts the system's state at future timepoints using differentiable ODE solvers and fully convolutional neural networks parametrizing the system's dynamics. We empirically evaluate GrINd on the DynaBench benchmark dataset, comprising six different physical systems observed at scattered locations, demonstrating its state-of-the-art performance compared to existing models. GrINd offers a promising approach for forecasting physical systems from sparse, scattered observational data, extending the applicability of deep learning methods to real-world scenarios with limited data availability.

\end{abstract}
\begin{figure}[ht]
\includegraphics[width=1\textwidth]{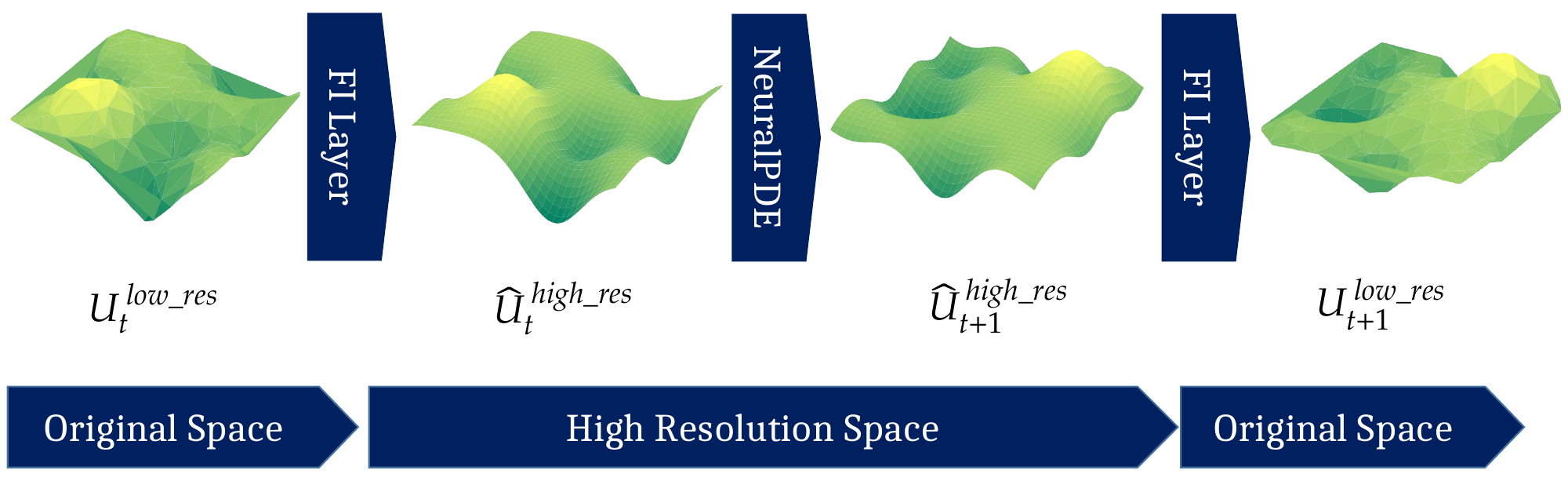}
\caption{Summary of our approach. The low-resolution observations are first mapped onto a high-resolution grid using a Fourier Interpolation Layer. In this high resolution space a predictive model (NeuralPDE) forecasts the evolution of the system which is then mapped back to the original observation space.}
\label{fig:figure1}
\end{figure}
\section{Introduction}
Understanding and accurately predicting the evolution of spatiotemporal physical systems is an important challenge in various scientific domains, encompassing fields such as climate modeling, weather forecasting, computational fluid simulation, biology, and many more~\cite{Kleinstreuer2010,Cullen1997,Bauer2015,Boussard2023}.
These systems are typically guided by physical laws summarized as partial differential equations, which describe the rate of change of the systems' quantities as a function of their partial derivatives in space~\cite{Cannon1967}.
Traditionally, numerical simulations have served as the cornerstone of predictive modeling for these systems, with various methods such as the method of lines, spectral methods, finite volume methods, and the finite element method designed to tackle specific challenges and specific classes of equations~\cite{larsson2003,Schiesser1991}.

In recent years, a notable paradigm shift has emerged, with a range of new deep learning-based methods being proposed to replace or augment the classical numerical methods.
Models like Physics Informed Neural Networks~\cite{Raissi2017a,Raissi2017b} can be used to solve differential equations by employing neural networks as solution approximators.
Other approaches like SINDy~\cite{Brunten2016} or PDE-Net~\cite{Long2019} solve the task of reconstructing the equation that describes the evolution of the system by using symbolic regression and sparsity enforcement.
Yet another range of methods focus solely on forecasting the evolution of the system. 
Models like PanguWeather~\cite{Bi2023}, FourCastNet~\cite{Kurth2022} or GraphCast~\cite{Lam2023} leverage large-scale neural architectures to learn the behavior of the system directly from data without any physical information.

Crucially, the performance of these types of models depends on the availability of large amounts of high-quality data that in most cases is assumed to be distributed across a spatial grid.
In fact most methods proposed for forecasting physical systems work only on grid structured data~\cite{Dulny2023} and a persistent challenge arises in scenarios where observational data is inherently sparse and spatially scattered, a common occurrence in many real-world applications.
As Dulny et al.~\cite{Dulny2023} pointed out, the task of predicting the evolution of a physical system based only on low-resolution scattered observations remains a challenging and unsolved task.

To tackle this challenge, we introduce GrINd (\textbf{Gr}id \textbf{I}nterpolation \textbf{N}etwork for Scattere\textbf{d} Observations), an abstract network architecture which leverages the high performance of grid-based models, by mapping the scattered observations onto a high-resolution grid using a Fourier Interpolation Layer.
In the high-resolution space a NeuralPDE-class model~\cite{Dulny2022} predicts the state of the system at future timepoints using a combination of differentiable ODE solvers~\cite{Chen2018} and the method of lines~\cite{Schiesser1991} with fully convolutional neural networks parametrizing the dynamics of the system.
Our approach is summarized in~\Cref{fig:figure1}.

We empirically evaluate the performance of our model on the DynaBench benchmark dataset~\cite{Dulny2023} encompassing six different physical systems observed on scattered locations and show its state-of-the-art performance compared to other non-grid based models.

Overall, our contributions can be summarized as follows:
\begin{enumerate}
    \item We introduce GrINd - a novel network architecture for forecasting physical systems from sparse, scattered observational data.
    \item We combine a Fourier Interpolation Layer and a NeuralPDE model for efficient predictions in high-resolution space
    \item We empirically evaluate our model on the DynaBench dataset, demonstrating state-of-the-art performance compared to existing models.
\end{enumerate}

We make the source code containing our model as well as all experiments in this paper publicly available.~\footnote{\repourl}

\section{Related Work}


Learning the behavior of real physical systems from empirical data using machine learning techniques has become of increasing interest in recent years.
The work in this area can generally be divided into two overlapping areas: data-driven solution and data-driven discovery of partial differential equations.
Examples of the former include data-driven models for predicting the evolution of various physical systems such as weather with transformer models~\cite{Bi2023,Kurth2022,Lam2023}, fluid simulation with graph neural networks~\cite{SanchezGonzalez2020} or simply solving specific PDEs using Physics Informed Neural Networks~\cite{Raissi2017a}.
Examples of the latter include SINDy~\cite{Brunten2016} which uses sparse regression to find the underlying governing equations of a system and PDE Net~\cite{Long2019} which uses a set of learnable filters and symbolic regression to reconstruct the equation.
The relevance of the area also emerges from the increasing number of benchmark datasets available for physical systems~\cite{Dulny2023,Rasp2023,Nathaniel2024}. 

Most works use high-resolution grid data which offers many advantages, allowing the use of well-developed models and architectures such as a modified vision transformer ~\cite{Bi2023,Kurth2022,Lam2023,Alkin2024,Li2023a,Lessig2023}.
Another approach by Ayed et al.~\cite{Ayed2019} uses a hidden-state neural-based model to forecast dynamical systems using a ResNet.
The Finite Volume Neural Network (FINN)~\cite{Praditia2021} predicts the evolution of diffusion-type systems by explicitly modeling the flow between grid points utilizing the Finite Volume Method.
NeuralPDE~\cite{Dulny2022} combines CNNs and the method of lines solver to find a solution for the underlying PDE.

Only a limited number of approaches have been proposed to tackle the challenging task of forecasting a physical system from scattered data.
Iakovlev et al.~\cite{Iakovlev2021} introduce continuous-time representation and learning of the dynamics of PDE-driven systems using a GNN.
The MGKN (multipole graph kernel network)~\cite{Li2020} proposed by Li et al. unifies GNNs with multiresolution matrix factorization to capture long-range correlations.
PhyGNNet~\cite{Jiang2023} divides the computational domain into meshes and uses the message-passing mechanism, as well as the discrete difference method, based on Taylor expansion and least squares regression, to predict future states.
Graph networks, however, can become computationally expensive when reaching a certain size, common in real-world tasks.

To the best of our knowledge, our approach is the first model to take advantage of Fourier Analysis to interpolate observations onto a higher resolution grid.
Several deep learning approaches have been proposed to perform training entirely in the Fourier domain such as Fourier Convolutional Neural Networks~\cite{pratt2017fcnn} or the Fourier Neural Operator~\cite{LiuSchiaffini2024}.
However, all works taking advantage of the Fast Fourier Transform (FFT) still require the input data to be grid-structured.

\section{Method}
In this section, we describe the technical details of our proposed approach.
The overall architecture of our approach consisting of the Fourier interpolation layer and the predictive NeuralPDE model can be seen schematically in~\Cref{fig:figure1}.

\subsection{Fourier Interpolation Layer}
\label{sec:fourier_interpolation_layer}
One of the core components of our approach is the Fourier interpolation layer, which takes as input the values of a function $u$ at several non-uniform locations $X=\{\mathbf{x}_1,\ldots,\mathbf{x}_H\}$ and outputs the interpolated values of the function $\Tilde{u}$ at grid locations $X_{grid}$.
To calculate the interpolation in a differentiable way, we leverage the approximation properties of the Fourier series summarized in~\Cref{theo:fourier_series}.

\begin{figure}[ht]
\includegraphics[width=1\textwidth]{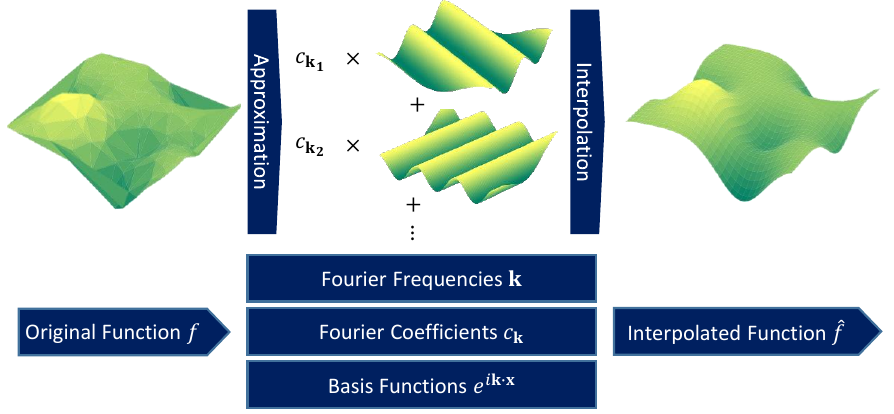}
\caption{Fourier Interpolation Layer. We use a Fourier series approximation of the original function $u$ by fitting the Fourier coefficients $c_{\mathbf{k}}$ using LLS. The Fourier coefficients can be used to evaluate the approximation at any arbitrary collection of points $\mathbf{x}$ thus making it possible to interpolate the function onto a high-resolution grid.}
\label{fig:fourier_layer}
\end{figure}

\begin{theorem}
\label{theo:fourier_series}
Given a periodic complex function $u\colon \Omega \rightarrow \mathbb{C}$ on the interval box $\Omega = [0, 1]^M$, such that $\int_{\Omega}||u||^2d\mathbf{x} < \infty$, the following series:
\begin{equation}
\label{eq:fourier_approximation}
    \hat{u}_{N}(\mathbf{x}) = \sum_{\mathbf{k}\in K} c_{\mathbf{k}}(u)e^{2i\pi\mathbf{k}\cdot\mathbf{x}}
\end{equation}

with 
\begin{align*}
    c_{\mathbf{k}}(u) &= \int_{\Omega}u(\mathbf{x})e^{2i\pi\mathbf{k}\cdot\mathbf{x}}d\mathbf{x} \\ 
    N &= (N_1, \ldots, N_M) \\
    K &= K_{N_1,\ldots,N_M} := K_{N_1}\times \ldots\times K_{N_M} \\
    K_{N_i} &= \begin{cases}
        \{-\frac{N_i}{2},\ldots,  \frac{N_i}{2}-1\} \text{ for } N_i \text{ even} \\ 
        \{-\frac{N_i-1}{2},\ldots,  \frac{N_i-1}{2}\} \text{ for } N_i \text{ odd}
    \end{cases}
\end{align*}
converges pointwise to $u$ as $N_i\rightarrow\infty$, i.e. 
\begin{equation*}
    \text{for each } \mathbf{x}\in[0, 1]^M\colon \lim_{\min{N_i}\rightarrow\infty} \hat{u}_{N}(\mathbf{x}) = u(\mathbf{x})
\end{equation*}
\end{theorem}
%
%
\begin{proof}
The proof of this theorem can be found in~\cite{Serov_2017} or~\cite{Potts2021}.
\end{proof}

This theorem underpins various numerical methods and techniques, including the Fast Fourier Transform (FFT), which plays an important role in signal processing~\cite{Bi_Zeng_2004}, data compression~\cite{redinbo2000}, and solving differential equations efficiently~\cite{Kopriva_2009}. 
However, applying FFTs directly to non-uniformly sampled data presents challenges, as the traditional approach assumes data points are uniformly distributed along the domain~\cite{shih2021,BARNETT20211}. 
To address this limitation, we employ an estimation technique to compute the Fourier coefficients $c_{\mathbf{k}}$ for non-uniformly sampled points as summarized in~\Cref{fig:fourier_layer}. 

Specifically, given the Fourier series expressed as~\Cref{eq:fourier_approximation} and the function $u$ measured at several non-uniform locations $X=\{\mathbf{x}_1,\ldots,\mathbf{x}_H\}$ we are trying to estimate the coefficients $c_{\mathbf{k}},\mathbf{k}\in K$ such that the approximation error 
\begin{equation}
    \label{eq:fourier_approximation_error}
    \sum_{\mathbf{x}\in X} ||u(\mathbf{x}) - \hat{u}_{N}(\mathbf{x})||^2
\end{equation}
is minimized.
$N = (N_1, \ldots, N_M)$ represents the number of frequencies in each dimension.
We note that while the number of frequencies can be chosen separately for each dimension $1,\ldots, M$ in practice we keep this parameter constant as $N_1 =\ldots= N_M =\colon N_{freq}$.

We formulate the estimation process as a linear least squares (LLS) problem where the coefficients $c_{\mathbf{k}}$ are the unknowns. 
The LLS problem can be expressed as follows:

\begin{equation}
\label{eq:optimization_problem}
    \min_{\mathbf{c}\in\mathbb{R}^{N_1\ldots N_M}}||\mathbf{A}\mathbf{c}-\mathbf{b}||^2
\end{equation}





where, $\mathbf{A}$ is a matrix representing the basis functions $e^{2i\pi\mathbf{k}\cdot\mathbf{x}}$ evaluated at the non-uniform sample points $\mathbf{x}$, i.e.
$A = [a_{jl}] = [e^{2i\pi\mathbf{k}_l\cdot\mathbf{x}_j}]_{\mathbf{k}_l\in K, \mathbf{x}_j\in X}$, 
$\mathbf{c}$ is a vector containing the unknown Fourier coefficients:
$\mathbf{c} = [c_{\mathbf{k}_l}]_{\mathbf{k}_l\in K}$
and $\mathbf{b}$ is a vector containing the values of the function $u$ evaluated at the non-uniform sample points $\mathbf{x}$:
$\mathbf{b} = [u(\mathbf{x}_j)]_{\mathbf{x}_j\in X}$

Solving this linear least squares system allows us to estimate the Fourier coefficients efficiently, enabling the interpolation $\Tilde{u}$ of the function $u$ onto a uniform grid of high resolution in Fourier space. 
We achieve this by evaluating the Fourier series approximation given by~\Cref{eq:fourier_approximation} at new high-resolution grid location points $\mathbf{x}\in X_{grid}$ as follows:

Assuming $\mathbf{c^*}$ is the solution to the LLS problem given in~\Cref{eq:optimization_problem}, we calculate the interpolated values at $\mathbf{x}\in X_{grid}$ as
\begin{equation}
    \Tilde{u}(\mathbf{x}) = \Re(\sum_{\mathbf{k}\in K} c^*_{\mathbf{k}}e^{2i\pi\mathbf{k}\cdot\mathbf{x}})
\end{equation}
where $\Re(z) = a$ is the real part of a complex number $z=a+bi\in\mathbb{C}$ .

This procedure originally only applies to the interpolation of periodic functions.
In~\Cref{app:fi_periodic_extension} we show how our approach can be extended to non-periodic functions.

Overall we call this interpolation procedure the $\operatorname{FI}$ (Fourier Interpolation) Layer which in summary takes as input a set of arbitrary locations $X\subset\mathbb{R}^M$ together with the corresponding values $u(X)$ of a function $u$ and a set of target locations $X_{\mathrm{target}}\subset\mathbb{R}^M$, transforms them into the Fourier coefficients $[c_\mathbf{k}]_{\mathbf{k}\in{K}}$.
These are then used to evaluate the Fourier approximation at the set of target locations~$X_{\mathrm{target}}$.
Overall this can be summarized as:

\begin{equation}
    \label{eq:fourier_interpolation_layer}
    \operatorname{FI}(X, u(X), X_{\mathrm{target}})\xrightarrow{\min_{\mathbf{c}\in\mathbb{R}^{|K|}}||\mathbf{A}\mathbf{c}-\mathbf{b}||^2}\mathbf{c}^*\xrightarrow{\Re(\sum_{\mathbf{k}\in K} c^*_{\mathbf{k}}e^{2i\pi\mathbf{k}\cdot\mathbf{x}})}\Tilde{u}(X_{\mathrm{target}})
\end{equation}

\subsection{NeuralPDE}
NeuralPDE were originally proposed by Dulny et al.~\cite{Dulny2022} for modeling dynamical systems, particularly those governed by partial differential equations (PDEs).
The core idea lies in representing the underlying PDE using the Method of Lines, which discretizes the spatial dimensions and transforms the PDE into a system of ordinary differential equations (ODEs).
This discretization allows spatial derivatives to be represented as convolutional operations, making CNNs a natural choice for parametrization, which are then trained using a differentiable ODE solver~\cite{Chen2018}.

Specifically, given measurements at grid locations $\mathbf{U}(t) = [U_{ij}(t)]$ of an underlying physical system $u$ which is assumed to be governed by some partial differential equation:

\begin{equation}
\label{eq:governing_equation}
    \frac{\partial u}{\partial t}=F(u, \nabla u, \nabla^2 u,\ldots)
\end{equation}
the task is to learn the function $F$ from data and use it to forecast the state of the system at future timepoints $t+1,\ldots, t+T$.

The method of lines treats the system $\mathbf{U}$ as a collection of ordinary differential equations (ODEs), with the value of the system at each of the grid locations as one variable in the system.
The governing~\Cref{eq:governing_equation} then becomes an ODE

\begin{equation}
\label{eq:governing_equation_ode}
    \frac{\diff\mathbf{U}}{\diff t} = \Tilde{F}(\mathbf{U})
\end{equation}
with all partial derivatives $\nabla u, \nabla^2 u,\ldots$ replaced by finite difference approximations, at grid locations, e.g. $\frac{\partial U_{ij}}{\partial x}\sim U_{ij+1}-U_{ij-1}$.

The function $\Tilde{F}$ can now be parametrized using a convolutional network $\theta$, as every finite difference approximation can be learned using a convolution operator~\cite{Dulny2022}, e.g. $\frac{\partial U_{ij}}{\partial x}\sim U_{ij+1}-U_{ij-1}$ corresponds to the convolution $\left[\begin{array}{rrr}
        0 & 0 & 0 \\ 
        -1 & 0 & 1 \\
        0 & 0 & 0 \\ 
    \end{array}\right]$.

The NeuralPDE model $\operatorname{N_{\theta}}$ is trained using a differentiable ODE solver~\cite{Chen2018} to forecast the physical system $u$.

\subsection{GrINd}
The overall architecture of our model is displayed in~\Cref{fig:figure1}.
We combine two Fourier Interpolation Layers with a spatiotemporal prediction model on gridded data to form the \textbf{Gr}id \textbf{I}nterpolation \textbf{N}etwork for Scattere\textbf{d} Observations (GrINd).
We select NeuralPDE~\cite{Dulny2022} as the spatiotemporal prediction model, as it showed best performance on the DynaBench Benchmark~\cite{Dulny2023}.

We first use a Fourier Interpolation Layer $\operatorname{FI}_1$ to transform the scattered, low-resolution observations $U_t^{\mathrm{low}}$ at locations $X_{\mathrm{low}}$ at time $t$ into a high-resolution grid interpolation $\hat{U}_t^{\mathrm{high}}$ at grid locations $X_{\mathrm{high}}$.
In the high-resolution space we use the NeuralPDE model $\operatorname{N_{\theta}}$ to forecast the state of the system at time $t+1$ as $\hat{U}_{t+1}^{\mathrm{high}}$.
Finally we transform the high resolution predictions back into the original space using another Fourier Interpolation Layer $\operatorname{FI}_1$ to obtain the final predictions $U_{t+1}^{\mathrm{low}}$.
Overall this can be written as:
\begin{equation}
\label{eq:grind}
    \operatorname{GrINd}(U_t^{\mathrm{low}}) = \operatorname{FI}_2(X_{\mathrm{high}},\operatorname{N_{\theta}}(\operatorname{FI}_1(X_{\mathrm{low}}, U_t^{\mathrm{low}},X_{\mathrm{high}})), X_{\mathrm{low}})
\end{equation}

\section{Experiments}
In this section we describe the experimental setup used for evaluating the performance of our proposed model.
\subsection{Data}
For our experiments we use the DynaBench~\cite{Dulny2023} dataset proposed by Dulny et al., which serves as a standardized benchmark for evaluating machine learning models on dynamical systems with sparse observations.
The dataset consists of six physical systems described by partial differential equations: advection, Burgers', Gas Dynamics, Kuramoto-Sivashinsky, Reaction-Diffusion and the Wave equation.
The equations have been simulated on a high-resolution grid with periodic boundary conditions using the finite differences method~\cite{leveque2007} and recorded at spatially scattered locations.
This mimics a real-world setting where measurement stations are typically also sparse and not located on a grid.

For our experiment we use the same version of the dataset on which the original benchmark experiments were performed~\cite{Dulny2023}.
It contains \num{7000} simulations of each physical system with varying initial conditions (splitted into \num{5000} train, \num{1000} validation and \num{1000} test simulations), each recorded over~\num{201} time steps at \num{900} measurement locations.

For evaluating the performance of our proposed Fourier Interpolation layer, we use the raw high-resolution version of the dataset as ground truth for comparison with the output of the interpolation layer.
Similar to Dulny et al., we also compare our results to models trained and evaluated on a version of the dataset containing observations on a $30\times 30$ grid.

\subsection{Baseline Models}
\label{sec:baseline_models}
We compare the performance of our proposed approach to the original models evaluated in the DynaBench benchmark~\cite{Dulny2023}.
These include four graph- and point-cloud based models: Point Transformer~\cite{Zhao2021}, Point GNN~\cite{Shi2020}, Graph Kernel Network~\cite{Anandkumar2019} and Graph PDE~\cite{Iakovlev2021}.
In our results we omit three models included in the original benchmark as they are strictly outperformed by all other models.

The baselines also include two models designed to work on grid data: NeuralPDE~\cite{Dulny2022} and a Residual Network~\cite{He2016}.
We omit the CNN baseline included in the original benchmark as it showed subpar performance.

Additionally we include the \textit{persistence} baseline from the original benchmark which corresponds to predicting no change of the system at all.

\subsection{Model Configuration}
We use $N_{freq}=18$ Fourier Frequencies in each Fourier Interpolation Layer, as this values shows the lowest interpolation error on most physical systems from the DynaBench dataset.
The target interpolation points are distributed along a $64\times 64$ grid, which is the original resolution of the simulations in the DynaBench dataset.

The NeuralPDE model architecture is taken from~\cite{Dulny2023}.

\subsection{Training}
We replicate the original training setting from the DynaBench benchmark as closely as possible.
The models are trained to predict the next simulation step and evaluated in a closed-loop rollout on 16 simulation steps.
We train the models with the Mean Squared Error loss using the Adam Optimizer~\cite{KingBa15} with a learning rate of ~\num{0.0001}.

All models are trained for a maximum of \num{100000} optimization steps while monitoring the 16 step rollout MSE on a separate validation dataset every \num{1000} steps with early stopping if the metric did not improve after 5 consecutive checks.
Afterwards the best checkpoint with respect to the validation metric is used for testing.

We implemented our models using PyTorch~\cite{NEURIPS2019_9015} and Pytorch Lightning~\cite{Falcon_PyTorch_Lightning_2019}.
The source code of our models and the performed experiments is available under \repourl.

\begin{figure}[t]
\includegraphics[width=1\textwidth]{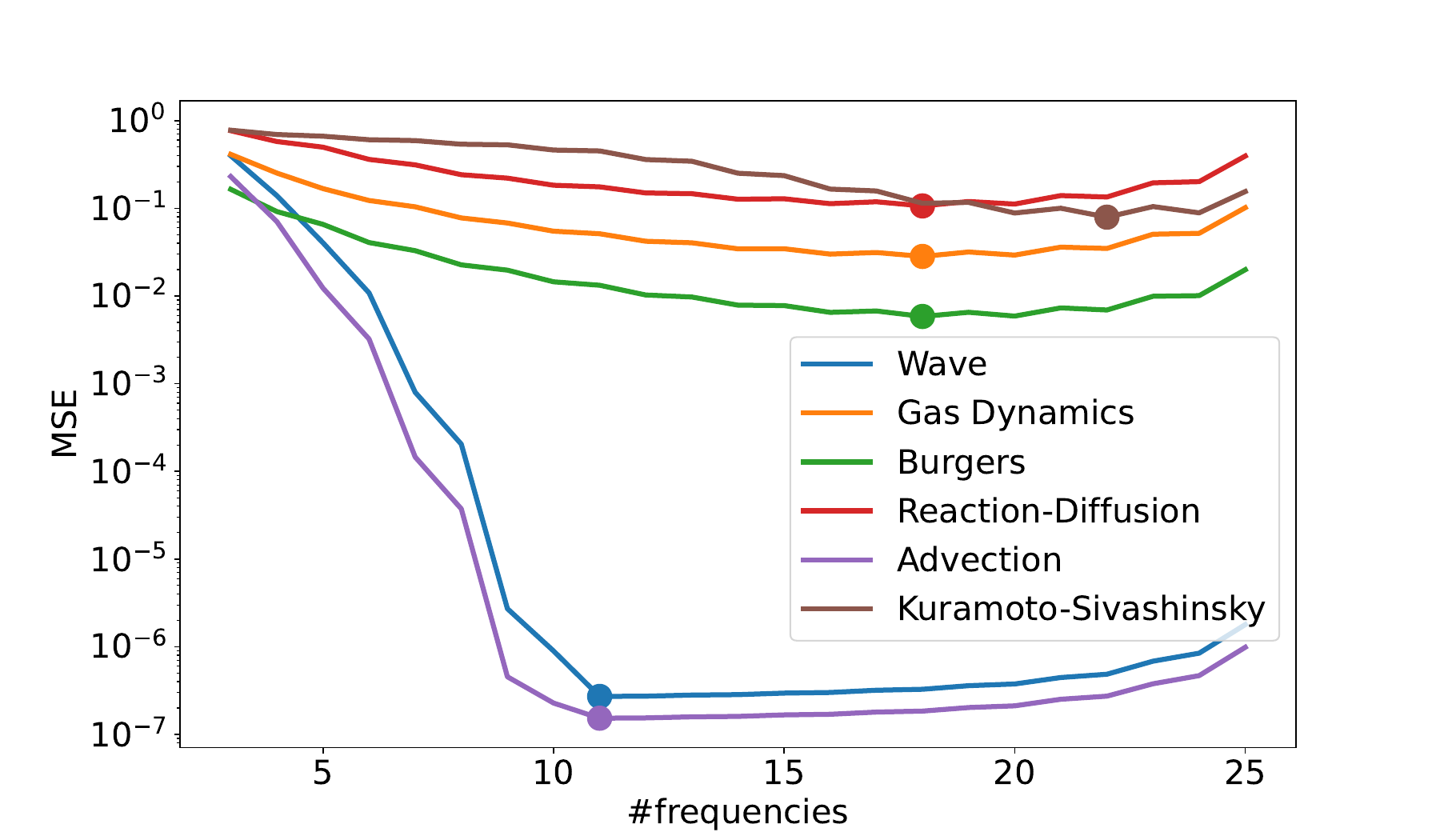}
\caption{Results of the interpolation experiment. The plot shows the Mean Squared Error between the interpolation given by our Fourier Interpolation Layer and the full resolution dataset as a function of the number of fourier frequencies used. The lowest interpolation error for each dataset has been marked with a dot.}
\label{fig:interpolation_results}
\end{figure}

\section{Results and Discussion}
In this section we present the experiments we performed and discuss the results.

\subsection{Interpolation accuracy}
\label{sec:interpolation_experiment}

Firstly, to showcase the feasibility of our proposed approach, we evaluate the performance of our proposed Fourier Interpolation Layer on the task of directly interpolating a set of scattered observations onto a high resolution grid.
To this end, we use the DynaBench dataset with \num{900} measurement locations as input to a single Fourier Interpolation Layer with interpolation target points distributed on a $64\times 64$ grid over the simulation domain. 
We compare the output of the Fourier Interpolation Layer to the \textit{full} resolution version of the DynaBench dataset which was used to sample the scattered observations.
We quantify the interpolation error using the \textit{Mean Squared Error} (MSE) metric.
The experiment is repeated for different selection of frequencies $N_{freq}$, varying between \num{3} and \num{25} (cf.~\Cref{sec:fourier_interpolation_layer}).

\Cref{fig:interpolation_results} illustrates the outcomes of the interpolation experiment conducted across six distinct physical systems within the DynaBench dataset. 
Across all systems, a consistent trend emerges, characterized by a distinctive "U" shape in the performance curves.
This trend is intuitively comprehensible: when a limited number of Fourier frequencies are chosen, the Fourier approximation fails to capture the intricate nature of the function adequately. 
Conversely, employing a high number of Fourier frequencies leads to overfitting of the Fourier Interpolation Layer, resulting in elevated reconstruction errors.

Furthermore, it becomes apparent that the optimal number of frequencies varies depending on the complexity of the system, suggesting that individual selection for each dataset is ideal. 
Nevertheless, the results show that for three out of the six systems, the optimal number of frequencies is 18. 
Therefore, we suggest this number as practical out-of-the-box choice of this hyperparameter and use it for our experiments on the DynaBench datasets.

\begin{figure}[t]
 \centering
\includegraphics[width=1\textwidth]{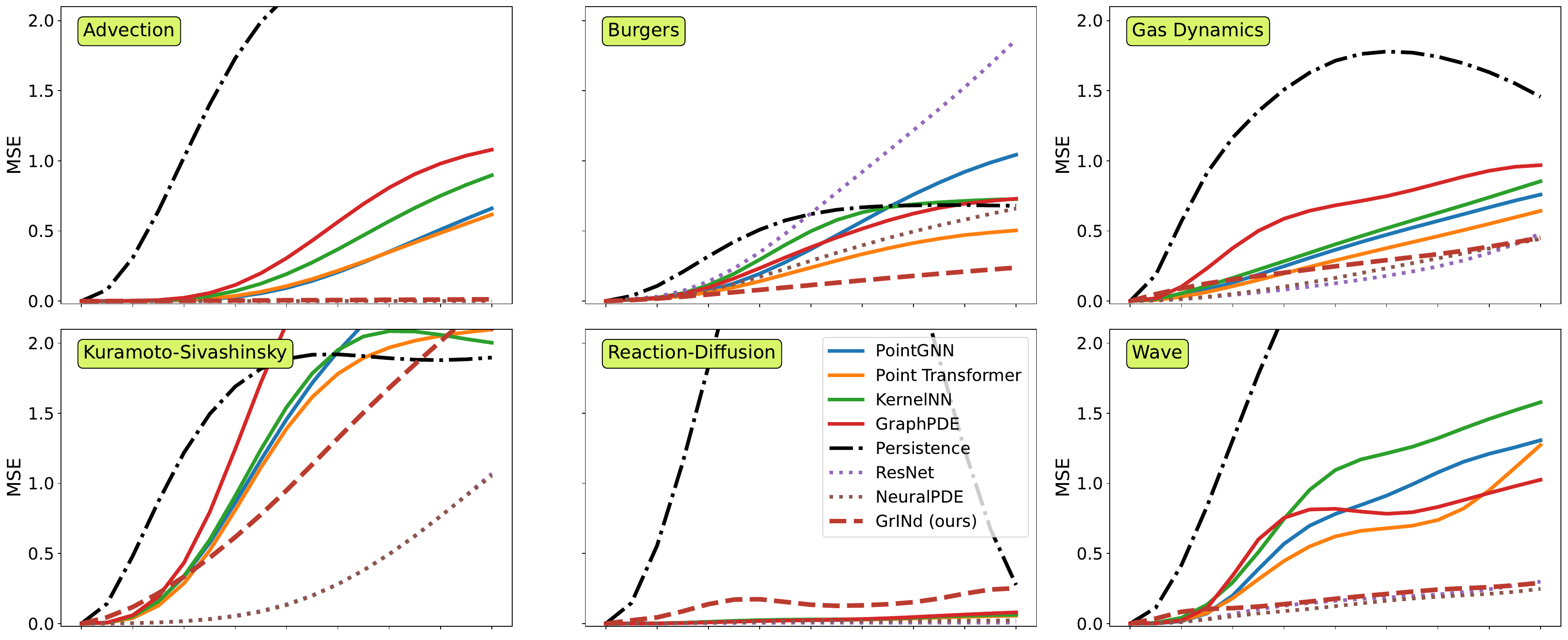}
\caption{Rollout results on the DynaBench dataset for 16 steps. The graph and point cloud baselines are plotted with solid lines, the grid baselines with dotted lines, the persistence baseline with a dash-dotted line and our model with a dashed line. Results of our model have been averaged over 10 runs. Baseline results taken from~\cite{Dulny2023}. For better readability MSE scores over 2.0 are not displayed.}
\label{fig:rollout_results}
\end{figure}

\subsection{DynaBench}
In this section, we present the results of our proposed GrINd model on the DynaBench dataset, comparing its performance against the baseline models described in~\Cref{sec:baseline_models}. 

\Cref{fig:rollout_results} shows the rollout results of our model over 16 steps compared to the baselines from the DynaBench dataset.
Overall, with the exception of the \textit{Reaction-Diffusion} and \textit{Kuramoto-Sivashinsky} equations, our model outperforms all models that work on scattered data (PointGNN, Point Transformer, KernelNN and GraphPDE) by a large margin.
This is especially true over longer prediction horizons.

In case of the \textit{Advection}, \textit{Gas Dynamics} and \textit{Wave} equation, our model performs on par with grid models, which get access to a grid structured array of observation locations.
While the grid models still get access to the same amount of measurement location ($30\times 30$ grid vs $900$ scattered observations), this additional structure has been shown to be beneficial for predicting the evolution of a physical system~\cite{Dulny2023}.

Finally, in case of the \textit{Burgers'} equation our model even outperforms the best grid model from the DynaBench benchmark.
This is possible because our model interpolates the system onto a higher-resolution grid than was used to train the baseline models.
We hypothesize that in this case the performance gain from using a higher grid outweighs the performance loss from the interpolation error (cf.~\Cref{sec:interpolation_experiment}).

\begin{table}[ht]
    \caption{MSE after 1 prediction step. The best perfoming model for each equation has been \underline{underlined}. Additionally, the best non-grid model has been \uwave{underwaved}. A = Advection, B = Burgers', GD = Gas Dynamics, KS = Kuramoto-Sivashinsky, RD = Reaction-Diffusion, W = Wave. Results of our model have been averaged over 10 runs. Baseline results taken from~\cite{Dulny2023}. }
    \centering
    \begin{tabular}{l@{\hskip 0.25cm}l@{\hskip 0.25cm}l@{\hskip 0.25cm}l@{\hskip 0.25cm}l@{\hskip 0.25cm}l@{\hskip 0.25cm}l}
    \toprule
    model &  A &  B &  GD &  KS &  RD &     W \\
    \midrule
    GrINd(ours)     &    $1.92\cdot 10^{-4}$ & $\uwave{6.56\cdot 10^{-3}}$ &       $4.93\cdot 10^{-2}$ &            $4.54\cdot 10^{-2}$ &            $2.33\cdot 10^{-2}$ & $3.61\cdot 10^{-2}$ \\
    \midrule
    GraphPDE          &   $1.37\cdot 10^{-4}$ & $1.07\cdot 10^{-2}$ &      $1.95\cdot 10^{-2}$ &              $7.20\cdot 10^{-3}$ &            $1.42\cdot 10^{-4}$ & $2.07\cdot 10^{-3}$ \\
    KernelNN          &   $6.31\cdot 10^{-5}$ & $1.06\cdot 10^{-2}$ &      $1.34\cdot 10^{-2}$ &              $6.69\cdot 10^{-3}$ &            $1.87\cdot 10^{-4}$ & $5.43\cdot 10^{-3}$ \\
    Point TF &   $4.42\cdot 10^{-5}$ & $1.03\cdot 10^{-2}$ &      $\uwave{7.25\cdot 10^{-3}}$ &              $\uwave{4.90\cdot 10^{-3}}$ &            $1.41\cdot 10^{-4}$ & $2.38\cdot 10^{-3}$ \\
    PointGNN          &   $\uwave{2.82\cdot 10^{-5}}$ & $\underline{8.83\cdot 10^{-3}}$ &      $9.02\cdot 10^{-3}$ &              $6.73\cdot 10^{-3}$ &            $\uwave{\underline{1.36\cdot 10^{-4}}}$ & $\uwave{\underline{1.39\cdot 10^{-3}}}$ \\
    \midrule
    NeuralPDE         &   $\underline{8.24\cdot 10^{-7}}$ & $1.12\cdot 10^{-2}$ &      $3.73\cdot 10^{-3}$ &              $5.37\cdot 10^{-4}$ &            $3.03\cdot 10^{-4}$ & $1.70\cdot 10^{-3}$ \\
    ResNet            &   $2.16\cdot 10^{-6}$ & $1.48\cdot 10^{-2}$  &      $\underline{3.21\cdot 10^{-3}}$ &              $\underline{4.90\cdot 10^{-4}}$ &            $1.57\cdot 10^{-4}$ & $1.46\cdot 10^{-3}$ \\
    \midrule
    Persistence       &   $8.12\cdot 10^{-2}$ & $3.68\cdot 10^{-2}$ &      $1.87\cdot 10^{-1}$ &              $1.42\cdot 10^{-1}$ &            $1.47\cdot 10^{-1}$ & $1.14\cdot 10^{-1}$ \\
    \bottomrule
    \end{tabular}

    \label{tab:results_single}
\end{table}

\textbf{Single Prediction Steps} 
We additionally numerically evaluate the performance of our model both on single and multi-step predictions.
\Cref{tab:results_single} displays the Mean Squared Error (MSE) after a single prediction step for each physical system in the DynaBench dataset. 
While in this case our GrINd model still achieves competitive performance, it is however outperformed by other grid and non-grid based models with the exception of the \textit{Burgers'} equation.
We observe that in this case the numerical interpolation error, as examined in~\Cref{sec:interpolation_experiment} is larger than the numerical forecasting error for the baseline models.
We hypothesize that the gain from mapping the observations onto a high-resolution grid is not sufficient to outweigh the interpolation error.
This is also supported by the conclusion from the original DynaBench benchmark, which suggests that for shorter prediction horizons grid models perform on par with non-grid models.

\begin{table}[ht]
    \caption{MSE after 16 prediction steps. The best perfoming model for each equation has been \underline{underlined}. Additionally, the best non-grid model has been \uwave{underwaved}. A = Advection, B = Burgers', GD = Gas Dynamics, KS = Kuramoto-Sivashinsky, RD = Reaction-Diffusion, W = Wave. Results of our model have been averaged over 10 runs. Baseline results taken from~\cite{Dulny2023}. }
    \centering
    \begin{tabular}{l@{\hskip 0.25cm}l@{\hskip 0.25cm}l@{\hskip 0.25cm}l@{\hskip 0.25cm}l@{\hskip 0.25cm}l@{\hskip 0.25cm}l}
    \toprule
     model &     A &  B & GD &   KS &   RD & W \\
    \midrule
    GrIND (ours)     &    $\uwave{1.23\cdot 10^{-1}}$ & $\uwave{\underline{2.38\cdot 10^{-1}}}$ &       $\uwave{4.51\cdot 10^{-1}}$ &            $2.33\cdot 10^{0}$ &            $2.53\cdot 10^{-1}$ & $\uwave{2.90\cdot 10^{-1}}$ \\
     \midrule
    GraphPDE          &   $1.08\cdot 10^{0}$ & $7.30\cdot 10^{-1}$ &      $9.69\cdot 10^{-1}$ &              $2.10\cdot 10^{0}$ &            $8.00\cdot 10^{-2}$ & $1.03\cdot 10^{0}$ \\
    KernelNN          &   $8.97\cdot 10^{-1}$ & $7.27\cdot 10^{-1}$ &      $8.54\cdot 10^{-1}$ &              $\uwave{2.00\cdot 10^{0}}$ &            $6.35\cdot 10^{-2}$ & $1.58\cdot 10^{0}$ \\
    Point TF &   $6.17\cdot 10^{-1}$ & $5.04\cdot 10^{-1}$ &      $6.43\cdot 10^{-1}$ &              $2.10\cdot 10^{0}$ &            $\uwave{5.64\cdot 10^{-2}}$ & $1.27\cdot 10^{0}$ \\
    PointGNN          &   $6.61\cdot 10^{-1}$ & $1.04\cdot 10^{0}$ &      $7.59\cdot 10^{-1}$ &              $2.82\cdot 10^{0}$ &            $5.82\cdot 10^{-2}$ & $1.31\cdot 10^{0}$ \\
    \midrule
    NeuralPDE         &   $2.70\cdot 10^{-4}$ & $6.60\cdot 10^{-1}$ &      $\underline{4.43\cdot 10^{-1}}$ &              $\underline{1.06\cdot 10^{0}}$ &            $2.24\cdot 10^{-2}$ & $\underline{2.48\cdot 10^{-1}}$ \\
    ResNet            &   $\underline{8.65\cdot 10^{-5}}$ & $1.86\cdot 10^{0}$ &      $4.80\cdot 10^{-1}$ &              $1.07\cdot 10^{0}$ &            $\underline{7.05\cdot 10^{-3}}$ & $2.99\cdot 10^{-1}$ \\
    \midrule
    Persistence       &   $2.39\cdot 10^{0}$ & $6.79\cdot 10^{-1}$ &      $1.46\cdot 10^{0}$ &              $1.90\cdot 10^{0}$ &            $2.76\cdot 10^{-1}$ & $2.61\cdot 10^{0}$ \\
    \bottomrule
    \end{tabular}
    \label{tab:results_rollout}
\end{table}

\textbf{Multiple Prediction Steps}
We further assess the predictive capabilities of GrINd by evaluating its performance after 16 prediction steps.
As shown in Table 2, for longer prediction horizons GrINd consistently outperforms non-grid based models in terms of MSE.
For the \textit{Burgers'} equation, our model even outperforms the best grid baseline.
As opposed to single step predictions, longer prediction horizons require more numerical stability which is not guaranteed by models working on unstructured data.
In this case the benefit of using a grid representation outweighs the error associated with using the Fourier Interpolation Layer.
The results for 16 prediction steps highlight our model's robustness in capturing the dynamics of physical systems over longer forecasting horizons.

Overall, the results demonstrate that GrINd is a promising approach for forecasting physical systems from sparse, scattered observational data. 
This is especially the case for longer prediction horizons, where our model shows more numerical stability than other models able to work with scattered data, similar to grid models like NeuralPDE~\cite{Dulny2022} or ResNet~\cite{He2016}.

\section{Conclusion and Future Work}


In this paper, we introduced GrINd (Grid Interpolation Network for Scattered Observations), a novel approach designed to address the challenge of forecasting physical systems from sparse, scattered observational data. 
GrINd leverages the high performance of grid-based deep learning approaches for the task of forecasting physical systems from sparse and scattered data.
It combines a novel Fourier Interpolation Layer with a NeuralPDE model to predict the evolution of spatiotemporal physical systems. 

The empirical evaluations on the DynaBench benchmark dataset reveal that GrINd outperforms existing models, particularly for longer prediction horizons, showcasing its applicability to real-world scenarios.
Analysis shows the potential of using interpolation-based methods for modeling physical systems, as grid-based models show more numerical stability.
Our work also opens up several avenues for future research and development that we summarize in the following.

While we use the NeuralPDE model in our architecture, as it showed best perfomance in the DynaBench benchmark~\cite{Dulny2023}, in principle our approach allows for any grid-based model to be used in its placer.
Future investigations could explore the integration of other such models to further improve predictive performance.

Another promising direction for future work is using a permutation-invariant neural network such as Deep Sets~\cite{Zaheer2017} or Set Transformers~\cite{Lee2019} to predict the Fourier Coefficients.
This could improve the interpolation layer by enabling pre-training on a wide range of data.
Additionally, learning the interpolation layer parameters alongside the NeuralPDE model could lead to better alignment between the interpolation and prediction processes, enhancing overall forecasting accuracy.

Instead of interpolating the data onto a high-resolution grid before learning the dynamics, an alternative approach could be to directly learn the dynamics in Fourier space using neuralODEs~\cite{Chen2018}.
By operating directly in the frequency domain, the model may capture underlying patterns and dependencies more effectively, potentially improving forecasting accuracy.

In summary, GrINd represents a significant contribution in forecasting physical systems from sparse observational data, by leveraging the already high performance of grid models.
However, there are still ample opportunities for further research and development to enhance the approach and address additional challenges in predictive modeling for spatiotemporal physical systems.

\begin{credits}

\subsubsection{\discintname}
The authors have no competing interests to declare that are relevant to the content of this article.
\end{credits}
%
%
%

\bibliographystyle{splncs04}
\bibliography{references}

\end{document}


\title{GrINd: Supplementary Material}
%
%
\author{Andrzej Dulny\and
Paul Heinisch\and
Andreas Hotho\and
Anna Krause}
%
\authorrunning{Dulny et al.}
%
\institute{CAIDAS, University Würzburg, Germany\\
\email{andrzej.dulny@uni-wuerzburg.de}\\
\email{paul.heinisch@stud-mail.uni-wuerzburg.de}\\
\email{hotho@informatik.uni-wuerzburg.de}\\
\email{anna.krause@uni-wuerzburg.de}
}
\maketitle              
\appendix
\section{Extension of the FI Layer to non periodic functions}
\label{app:fi_periodic_extension}
Originally our FI Layer can only be applied to non periodic functions.
In practice however, our approach can be extended to any non periodic function $f$ as follows:
We define a surrogate periodic function $g\colon \Omega \rightarrow \mathbb{C}$ as
\begin{equation}
    g(\mathbf{x})=g(x_1, x_2,\ldots, x_M):=f(\phi(x_1),\ldots,\phi(x_M))
\end{equation}  
with
$\phi\colon\mathbb{R}\rightarrow\mathbb{R}$ defined as $\phi(x)=2|\frac{1}{2}-x|$.

\begin{figure}[ht]
\centering
\includegraphics[width=0.7\textwidth]{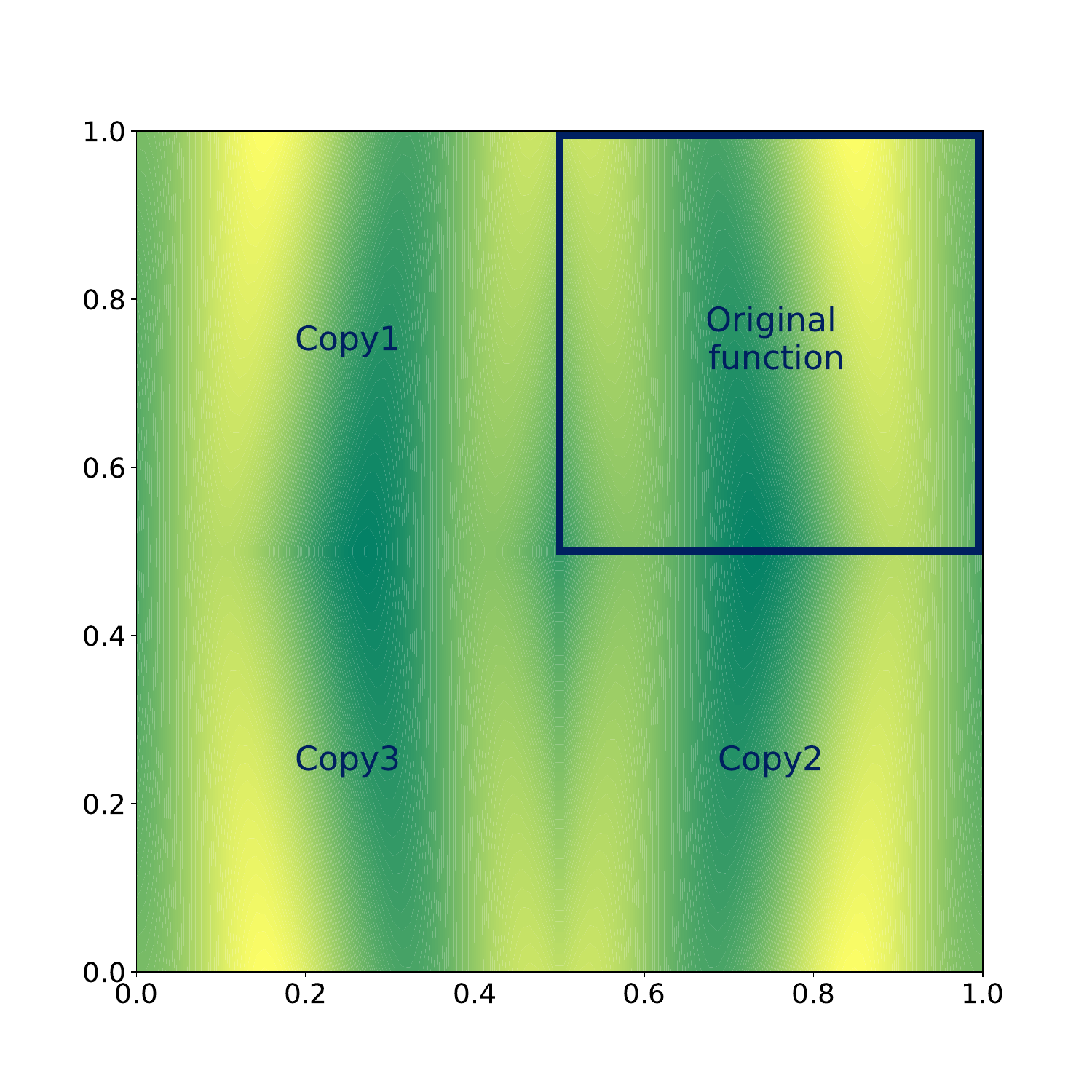}
\caption{Summary of our extension of the FI Layer to non periodic functions. The original function (top left) is compressed to the interval $[\frac{1}{2}, 1]$ in each dimension and then mirrored once in each dimension. The resulting function is periodic in each dimension and can be used as input to the FI Layer.}
\label{fig:non_periodic_extension}
\end{figure}

Intuitively, this corresponds to mirroring the function once in each dimension and concatenating the results along each dimension once as seen in~\Cref{fig:non_periodic_extension}.
This results in $2^M$ copies of the original function $f$ each flipped along one or more dimensions, which is periodic in the original domain. 
The original function $f$ can be recovered and interpolated in the subdomain $\Omega_{sub}=[\frac{1}{2}, 1]^M\subset\Omega$.

\section{Additional interpolation experiments}
We perform additional interpolation experiments on the DynaBench dataset, in which we compare the output of the Fourier Layer with a lower number of observation points (225 and 484) to the full resolution simulation.
The results are shown in~\Cref{fig:interpolation_results_low} and~\Cref{fig:interpolation_results_medium}.

\begin{figure}[ht]
\centering
\includegraphics[width=0.8\textwidth]{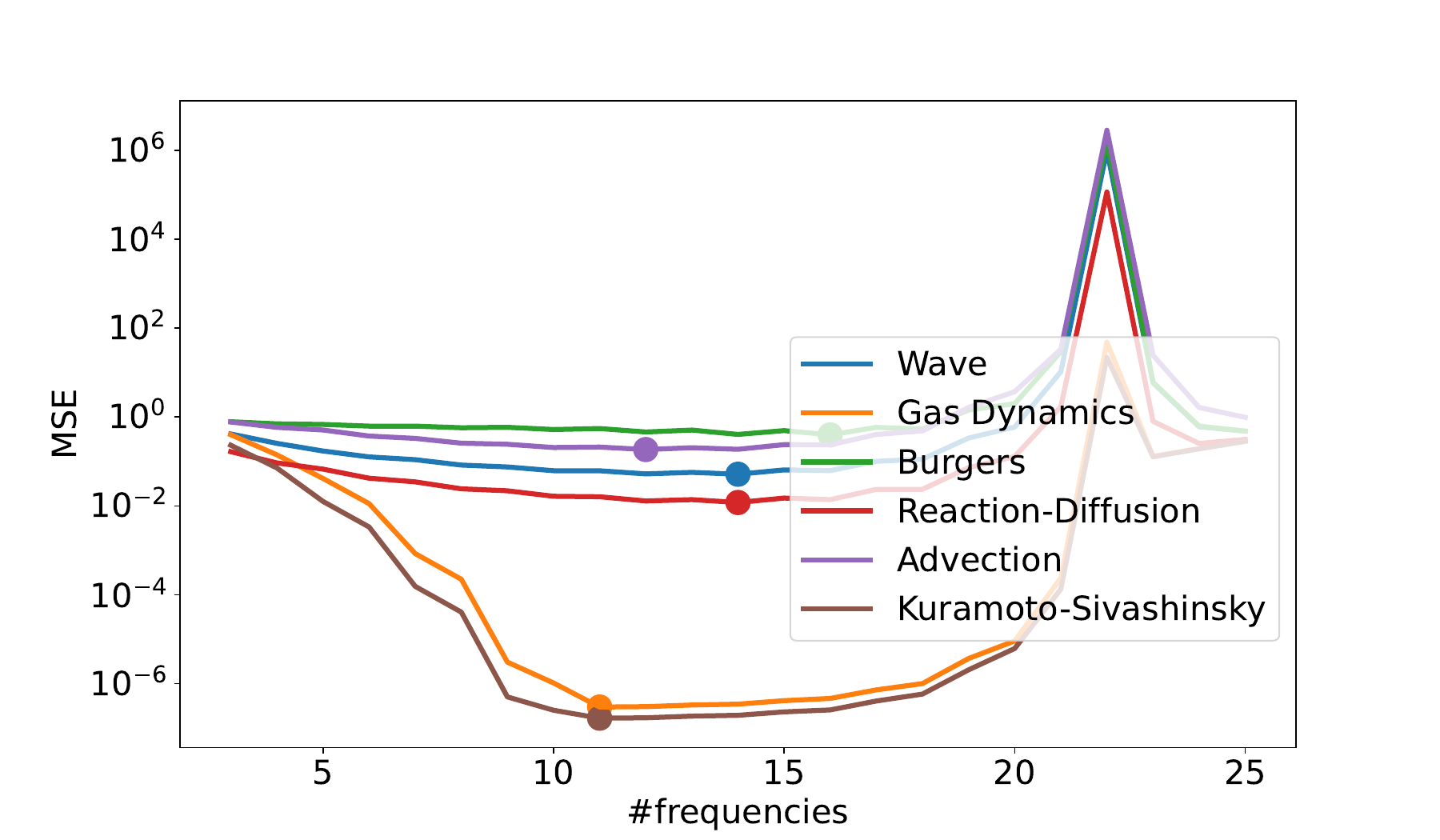}
\caption{Results of the interpolation experiment on the \textit{medium} version of the DynaBench dataset. The plot shows the Mean Squared Error between the interpolation given by our Fourier Interpolation Layer and the full resolution dataset as a function of the number of Fourier frequencies used. The lowest interpolation error for each dataset has been marked with a dot.}
\label{fig:interpolation_results_medium}
\end{figure}

\begin{figure}[ht]
\centering
\includegraphics[width=0.8\textwidth]{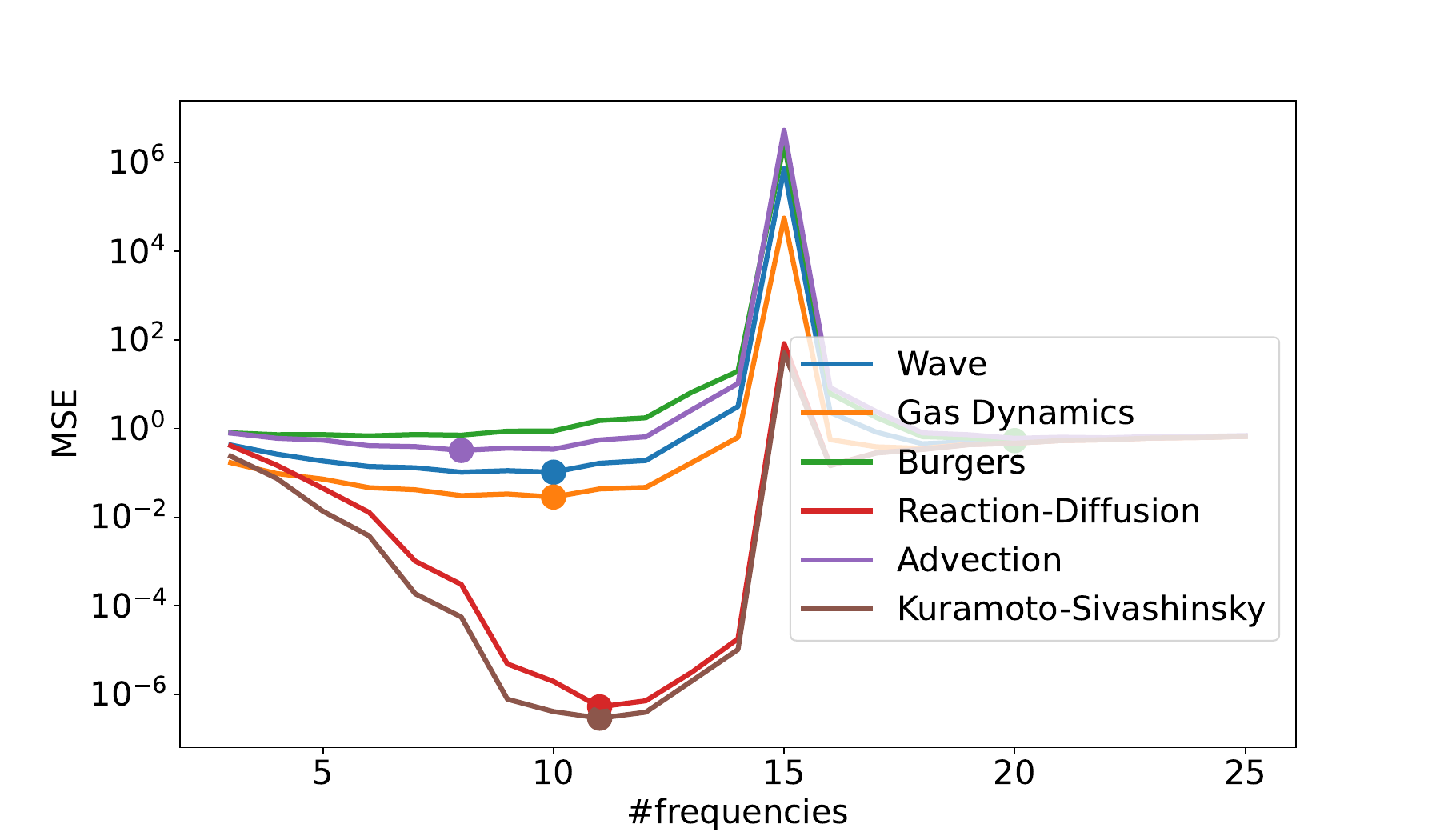}
\caption{Results of the interpolation experiment on the \textit{low} version of the DynaBench dataset. The plot shows the Mean Squared Error between the interpolation given by our Fourier Interpolation Layer and the full resolution dataset as a function of the number of Fourier frequencies used. The lowest interpolation error for each dataset has been marked with a dot.}
\label{fig:interpolation_results_low}
\end{figure}

Overall the interpolation results for the lower resolution data show that the optimal number of frequencies correlates with the number of observation points.
This optimum lies in the interval $[11, 16]$ for the medium resolution version and in the interval $[8, 11]$ for the low resolution version, depending on the type of physical system.
When using the Fourier Interpolation Layer for a given number $n$ of observation points this corresponds to the optimal number of frequencies of around $0.6\sqrt{n}$.

%
%
%
